\documentclass[conference,svgnames]{article}
\usepackage{fullpage}
\usepackage{authblk}
\usepackage[sortcites=true]{biblatex}

\usepackage[utf8]{inputenc}
\usepackage[T1]{fontenc}
\usepackage{hyperref}
\usepackage{booktabs}
\usepackage{amsfonts}
\usepackage[svgnames]{xcolor}
\usepackage{amsmath}
\usepackage{amsthm}
\usepackage{newtxmath}
\usepackage{subcaption}
\usepackage{threeparttable}
\usepackage{caption}
\usepackage{pifont}
\usepackage{mdwlist}
\usepackage{pgfplots}
\pgfplotsset{compat=1.18}

\theoremstyle{theorem}
\newtheorem{theorem}{Theorem}
\theoremstyle{lemma}
\newtheorem{lemma}[theorem]{Lemma}
\theoremstyle{corollary}
\newtheorem{corollary}[theorem]{Corollary}
\theoremstyle{definition}
\newtheorem{definition}[theorem]{Definition}
\theoremstyle{remark}
\newtheorem{remark}[theorem]{Remark}

\newcommand{\compind}{\stackrel{c}{\approx}}

\newcommand{\datauniverse}[0]{\mathcal{U}}
\newcommand{\hypspace}[0]{\mathcal{H}}
\newcommand{\trainset}[0]{\mathcal{D}}
\newcommand{\forgetset}[0]{\mathcal{D}_f}
\newcommand{\retainset}[0]{\trainset \setminus \forgetset}
\newcommand{\modelinit}[0]{\texttt{init}}
\newcommand{\learn}[0]{\texttt{learn}}
\newcommand{\unlearn}[0]{\texttt{unlearn}}
\newcommand{\infer}[0]{\texttt{infer}}
\newcommand{\cost}[0]{\texttt{cost}}
\newcommand{\util}[0]{\texttt{util}}
\newcommand{\origmodel}[0]{M_o}
\newcommand{\controlmodel}[0]{M_c}
\newcommand{\unlearnedmodel}[0]{M_u}
\newcommand{\adv}[0]{\mathcal{A}}
\newcommand{\chal}[0]{\mathcal{C}}
\newcommand{\kld}[0]{D_{\text{KL}}}
\newcommand{\cmark}{\color{Green}\ding{51}}
\newcommand{\xmark}{\color{Crimson}\ding{55}}

\title{Mirror Mirror on the Wall, Have I Forgotten it All?\\A New Framework for Evaluating Machine Unlearning}

\author[,1]{Brennon Brimhall$^\ast$}
\author[,1,2]{Philip Mathew\thanks{Equal contribution, listed alphabetically.}}
\author[1,2]{Neil Fendley}
\author[1]{Yinzhi Cao}
\author[1]{Matthew Green}

\affil[1]{Johns Hopkins University}
\affil[2]{Johns Hopkins University Applied Physics Laboratory}

\begin{document}

\maketitle

\begin{abstract}
    Machine unlearning methods take a model trained on a dataset $\trainset$ and a forget set $\forgetset$, then attempt to produce a model as if it had only been trained on $\retainset$. We empirically show that an adversary is able to distinguish between a mirror model (a control model produced by retraining without the data to forget) and a model produced by representative unlearning methods from the literature \autocite{2023FosterSSD, 2020Graves_Amnesiac, 2023Chundawat_bad_teaching, 2024Zhang_Certified_Deep_Unlearning}. We build distinguishing algorithms based on evaluation scores in the literature (i.e.\ membership inference scores) and Kullback-Leibler divergence.

    We propose a strong formal definition for machine unlearning called \emph{computational unlearning}. Computational unlearning is defined as the inability for an adversary to distinguish between a mirror model and a model produced by an unlearning method. If the adversary cannot guess better than random (except with negligible probability), then we say that an unlearning method achieves computational unlearning.

    Our computational unlearning definition provides theoretical structure to prove unlearning feasibility results. For example, our computational unlearning definition immediately implies that there are no deterministic computational unlearning methods for entropic learning algorithms. We also explore the relationship between differential privacy (DP)-based unlearning methods and computational unlearning, showing that DP-based approaches can satisfy computational unlearning at the cost of an extreme utility collapse. These results demonstrate that current methodology in the literature fundamentally falls short of achieving computational unlearning. We conclude by identifying several open questions for future work.
\end{abstract}

\section{Introduction}
Machine learning models require massive amounts of training data. Data is collected by scraping publicly available web content \autocite{2023David_OpenAICrawler, 2024Wired_AmazonInvestigating, 2024_AnthropicCrawler}, purchasing access to private databases \autocite{2024Knibbs_CondeNast, 2024OpenAI_CondeNast, 2024OpenAI_VoxMedia, 2024Verge_OpenAIVox, 2023AxelSpringer_OpenAI, 2023OpenAI_AxelSpringer, 2024Atlantic_OpenAI, 2024OpenAI_Atlantic}, and collecting data on their own to assemble training datasets \autocite{2022Schuhmann_laion5b, 2023Touvron_llama, 2020Brown_gpt3}. Due to the massive scale, datasets cannot be thoroughly vetted and may contain data that is copyrighted, inaccurate, protected, or contain otherwise undesirable information.

Legal protections exist for those who wish to protect their privacy, copyrighted content, and financial history in multiple countries. Examples include the EU GDPR (right to be forgotten) \autocite{2016EuropeanParliament_GDPR}, US DMCA (copyright infringement takedown) \autocite{1998USCongress_DMCA}, US FCRA (corrections to credit history) \autocite{1970USCongress_FCRA}, and US HIPAA (corrections to personal health data) \autocite{1996USCongress_HIPAA}. Specific instances of training data may also be illegal on their own: for example, it is illegal to possess child sexual abuse material (CSAM) in the US and in many other jurisdictions. Despite this, popular datasets \autocite{2022Schuhmann_laion5b} used to train models like Stable Diffusion contained illegal CSAM \autocite{2023Thiel_csam_in_laion}. Further, prior work has established the threat of data poisoning attacks that create undetectable backdoors in models \autocite{Goldwasser2022_backdoors}. This means that model data may be intentionally corrupted by an adversary.

These threats can be addressed by re-training the model from scratch without the offending data. However, since training large models is capitally and computationally intensive, a major area of interest is \emph{machine unlearning:} efficiently removing traces of the offending data, known as the as the \emph{forget set}, without training a new model from scratch \autocite{2015Cao_unlearning, 2016Abadi_DP_SGD, 2019GolatkarFF, 2020bourtoule_machine, 2020Graves_Amnesiac, 2021Gupta_adaptive_machine_unlearning, 2021Ullah_TV_Stability, 2022Nguyen_survey, 2023Chundawat_bad_teaching, 2023FosterSSD}.

\subsection{Our Contributions}
This work consists of three major contributions: (1) a new formal definition and framework for evaluating unlearning, (2) empirical results showing that many unlearning methods produce results that are distinguishable from a control, and (3) several theoretical implications of our framework.

\paragraph{Computational unlearning framework.} Our primary contribution is a new formal definition and framework for evaluating unlearning called \emph{computational unlearning} that we detail in \S\ref{sec:game}. In brief, computational unlearning tests the ability of an adversary to distinguish between a model produced by an unlearning method and a model trained from scratch with the forget set removed. If the adversary is only able to do so with negligible probability, then we say that the unlearning method achieves computational unlearning. Because the adversary is unable to distinguish between the control and unlearned models, it follows that all information about the forget set has been ``deleted'' by the unlearning method. The game is defined in both a white-box (i.e.\ adversary has full access to model parameters) and a black-box (i.e.\ adversary only has API access to model) setting.

\paragraph{Many unlearning methods do not achieve indistinguishability.} We construct two scoring methods $\texttt{MIAScore}$ and $\texttt{KLDScore}$ in \S\ref{sec:scores} which an adversary can use to distinguish between an unlearned model and a model that has never seen the forget set. We study previously proposed unlearning methods \autocite{2023FosterSSD, 2020Graves_Amnesiac, 2023Chundawat_bad_teaching, 2024Zhang_Certified_Deep_Unlearning} and show that each fail to achieve computational unlearning for ResNet-18 models \autocite{he2015deepresiduallearningimage} trained on CIFAR-10 \autocite{2009Krizhevsky_cifar} in \S\ref{sec:empirical}. We also experiment with how distinguishing rates are affected by the forget set size and unlearning method hyperparameters. 

\paragraph{Theoretical implications of computational unlearning.} We describe several implications of our computational unlearning framework in \S\ref{sec:theory}. We first show that any deterministic computational unlearning algorithms must achieve \emph{perfect unlearning} (i.e.\ it must produce the exact same model as retraining) and discuss implications for heuristic and certified removal unlearning methods. Second, we show that using differential privacy to achieve black-box computational unlearning leads to utility collapse (i.e.\ utility must be equivalent to a model that is randomly initialized).

\section{Background}
\label{sec:background}

Widespread interest in machine unlearning has led to a great deal of work in the area. We briefly review several approaches explored in the literature, extending the taxonomy proposed by Nguyen et al.~\autocite{2022Nguyen_survey}. We categorize machine unlearning method in one of three ways: as \emph{heuristic unlearning}, \emph{approximate unlearning}, or \emph{exact unlearning} as applied to classification models.

\paragraph{Heuristic unlearning.}
\label{sec:heuristic-examples}
Unlike exact and approximate unlearning methods, \emph{heuristic unlearning methods} do not have any formal guarantees. However, they are typically much less expensive than applying differential privacy or retraining the model~\autocite{2023FosterSSD, 2019GolatkarFF,2023Chundawat_bad_teaching,Kodge2023DeepUF,Tarun2021FastYE}. These rely on various heuristics that aim to minimize an unlearning ``score'' that that attempts to capture how well a machine learning model has forgotten. Membership inference attacks (MIA)~\autocite{2017Shokri_membership} are a popular scoring method used in the literature. 

We now describe three heuristic unlearning methods: \emph{bad teacher unlearning} \autocite{2023Chundawat_bad_teaching}, \emph{amnesiac unlearning} \autocite{2020Graves_Amnesiac}, and \emph{selective synaptic dampening (SSD)} \autocite{2023FosterSSD}. Each of these heuristic unlearning methods are evaluated on membership inference attack (MIA) scores; this is representative of many heuristic unlearning methods.

\begin{itemize}
    \item \emph{Bad teacher unlearning.} Bad teacher unlearning rests on the assumption that, after forgetting a data point, a model's behavior on that data point should be similar to that of a randomly initialized model. To forget $\forgetset$ the model is ``taught'' to reflect the behavior of a randomly initialized model (i.e.\ a \emph{bad teacher}).

    \item \emph{Amnesiac unlearning.} Amnesiac unlearning tries to reverse the changes to the model incurred by training on $\forgetset$ by keeping track of all batches containing elements from $\forgetset$; gradient \emph{ascent} is performed on these training batches at forget time. This attempts to ``backtrack'' towards a model that never had those gradient updates applied. We note that this approximates the approaches taken by many exact unlearning methods.

    \item \emph{Selective synaptic dampening (SSD).} SSD measures the $\forgetset$-related information in each neuron by using the Fisher information matrix (FIM). Neurons that contain lots of information about examples in $\forgetset$ are ``zeroed out'' by scaling down their weights. One can think of SSD as a pruning algorithm where ``branches'' of the network are ``removed'' based on their ``knowledge'' of $\forgetset$.
\end{itemize}

\paragraph{Approximate unlearning.}
\label{sec:approx-unlearning}
An \emph{approximate machine unlearning method} attempts to output a model that is approximately equal to a model trained without the forget set with high probability. Approximate machine unlearning methods are typically based on the notions of \emph{differential privacy}~\autocite{2014dwork_dp} and \emph{certified removal}~\autocite{2023Guo_Certified_Removal}.

\paragraph{Differential privacy.} Differential privacy~\autocite{2014dwork_dp} bounds the difference between outputs of a randomized algorithm on similar data sets. In the context of machine learning, this can be implemented as either (1) producing model parameters that are similar to the model parameters produced by training on a similar dataset or (2) producing an inference that is similar to the inference produced by a model trained on a similar dataset.

In the first case, we achieve a ``white-box'' differential privacy property because noise is incorporated into the model parameters during the training process. A practical example of such a learning method is \emph{differentially private stochastic gradient descent (DP-SGD)}~\autocite{2016Abadi_DP_SGD}.

In the second case, we achieve a weaker ``black-box'' differential privacy property because the noise is only integrated after training. Intuitively, this means a differentially private model has no guarantee to maintain privacy if you are given access to model parameters. Additionally, privacy budget is consumed at inference time; this forces an upper bound of the number of permitted queries.

\paragraph{Certified removal.} Certified removal draws inspiration from the aforementioned notion of differential privacy, extending a white box privacy guarantee to hold for a learning and unlearning method. Their aim is to bound the difference in model parameters produced by the unlearning method and the model parameters produced by the learning method without a particular data point in the training set:

\begin{definition}[$(\epsilon, \delta)-$Certified Removal \autocite{2023Guo_Certified_Removal}]
\label{def:certified-removal}
    $\learn, \unlearn$ satisfy $(\epsilon, \delta)-$certified removal if, for all $T \subseteq \hypspace$, $x \in \trainset \subseteq \datauniverse$,
    $$\mathbb{P}\left(\learn\left(\modelinit\left(1^\lambda\right), \trainset \setminus x\right) \in T\right) \leq e^{\epsilon}\mathbb{P}\left(\unlearn\left(\origmodel, \{x\}\right) \in T\right) + \delta
    $$
    and
    $$\mathbb{P}\left(\unlearn\left(\origmodel, \{x\}\right) \in T\right) \leq e^{\epsilon}\mathbb{P}\left(\learn\left(\modelinit\left(1^\lambda\right), \trainset \setminus x\right) \in T\right) + \delta
    $$ where $\origmodel = \learn\left(\modelinit\left(1^\lambda\right), \trainset \right)$.
\end{definition}

Note the similarity in this definition to the constraint imposed by differential privacy (DP). Indeed, Guo et al.\ mention that a DP learning algorithm is sufficient for $(\epsilon, \delta)-$certified removal. Certified removal extends the DP framework to a pair of methods $(\learn, \unlearn)$ instead of a single DP mechanism. As a result, certified removal allows for models with higher base performance, since less noise needs to be added during training than is required for DP.

Notably, Guo et al.~\autocite{2023Guo_Certified_Removal} introduced a \emph{Newton update removal mechanism} to remove information from linear models with convex objectives, which has been explored further by other works~\autocite{2023warnecke_featureslabels, 2022Mehta_DeepUV}. Zhang et al.~\autocite{2024Zhang_Certified_Deep_Unlearning} extend this to non-linear models with non-convex objectives via \emph{certified deep unlearning}. These methods allow practitioners to have some certification regarding how ``close'' the unlearned model is to a model retrained from scratch without the forget set.

Certified deep unlearning extends the Newton update removal mechanism (introduced in~\autocite{2023Guo_Certified_Removal}) to non-linear models with non-convex objective functions. Their approach resorts to using \emph{projected gradient descent (PGD)} \autocite{1997Bertsekas_PGD} for optimization, allowing for the guarantees introduced by Guo et al.\ \autocite{2023Guo_Certified_Removal} to be extended to these models.

Despite certified removal being defined as a \emph{white-box} property --- that is, an adversary seeing model parameters --- literature typically evaluates unlearning performance by computing membership inference attack scores (MIA)~\autocite{2017Shokri_membership} that do not have access to the model weights (i.e.\ a \emph{black-box} evaluation setting). We discuss certified removal methods in \S\ref{sec:empirical}.

\paragraph{Exact unlearning.}
An \emph{exact unlearning method} modifies the original model such that its outputs exactly match a model trained without the forget set. We are unaware of any exact unlearning method for neural networks that does not involve some degree of retraining. The most common approaches rely on saving checkpoints of model state at train time~\autocite{2020bourtoule_machine, 2021Ullah_TV_Stability}. Unlearning then consists of rewinding to a checkpoint that has not been influenced by the forget set and then resuming training from that point without the forgotten data. This technique is essentially a time-space tradeoff; multiple checkpoints of the model must be saved out during training. The worst-case retraining cost may be equivalent to retraining the model from scratch (for example, if the forget set contains an element from the first batch). We do not study exact unlearning in this work.

\paragraph{Unlearning evaluation.}

Many machine unlearning works attempt to justify their approach by optimizing some \emph{unlearning score}. Examples include accuracy gap scores and membership inference attacks:

\begin{itemize}
\item {\itshape Accuracy Gap Scores:} A popular approach in prior work is to demonstrate an ``accuracy gap'' between the unlearned model's performance when queried on the forget set (the set of examples to be forgotten) and the retain set (the training set minus the forget set) in a black-box manner~\autocite{2015Cao_unlearning, 2020bourtoule_machine, 2024Cha_instancewise_unlearning, 2019GolatkarFF}. The intuition here is that the unlearning method has ``remembered'' the retain set (because it maintains good accuracy) but ``forgotten'' the forget set (because it has lost accuracy on the items to be forgotten).

\item {\itshape Membership Inference Attack Scores:} Membership inference attacks, formalized in~\autocite{2017Shokri_membership}, are another common way to evaluate the performance of machine unlearning algorithms in literature. An attacker is given a data point and black-box access to a machine learning model; the attacker then attempts to predict if the data point was part of the training set. Some heuristic machine unlearning proposals are specifically designed to minimize these scores~\autocite{2020Graves_Amnesiac, 2023Chundawat_bad_teaching, 2023FosterSSD}.

\end{itemize}

Framing machine unlearning in a score-based manner is attractive: it provides an easy way to facilitate comparison, and it also satisfies intuitive beliefs about how the model should behave after unlearning. However, the use of score-based definitions does not fully capture the expectations of model behavior after unlearning.

Prior work has demonstrated multiple issues with unlearning evaluation:

\begin{itemize}
    \item {\itshape Over-Unlearning:} Hu et al.~\autocite{hu2024dutyforgetrightassured} considers the phenomenon of {\em over-unlearning}, where an attacker with black-box access to a model has the model provider run an unlearning method adversarially with the intent of reducing overall model utility. A key feature of this scenario is that an attacker can adaptively issue requests to forget data.
    
    \item {\itshape Unlearning Inversion:} Hu et al.~\autocite{hu2024learnwantunlearnunlearning} shows that many machine unlearning techniques are vulnerable to {\em unlearning inversion attacks}. In this scenario, an attacker with white-box access is able to recover an unlearned dataset by considering the differences between an original and unlearned model. This is similar in spirit to membership inference attacks.
    
    \item {\itshape Forgeability:} Recent work~\autocite{sekhari2021remember, 2022Thudi_Auditable} demonstrates that current machine unlearning definitions permit \emph{forgeability}. Forgeability is achieved if an adversary can \emph{forge} a proof that demonstrates their model was trained without the forget set despite being trained with it it. An adversary is permitted to order the retained data points in batches. This allows them to arrive at the same weights as a model trained on the forget set.
\end{itemize}

We believe that these issues fundamentally stem from the score-based approach to evaluating machine unlearning and that repairing this requires a fundamentally different formal definition of unlearning based on indistinguishability.

\paragraph{Machine unlearning is indistinguishability.}
We first motivate our desired functionality through the lens of the $k$-nearest neighbors ($k$-NN) machine learning algorithm. $k$-NN is a simple learning algorithm that memorizes every training example it is presented with. At inference time, the model finds the $k$ nearest training examples according to some metric and classifies based on the their round truth labels. One of the properties of $k$-NN is that it immediately admits an unlearning algorithm: simply delete the training examples you wish to forget. This produces a model that is indistinguishable from a control.

We claim that \emph{indistinguishability} represents a superior way to evaluate unlearning. This idea is not new and features in prior work~\autocite{2024Zhang_Certified_Deep_Unlearning, 2023Guo_Certified_Removal, 2023FosterSSD} but is not measured directly. We propose doing so here. In other words, no efficient (p.p.t., or probabilistic polynomial time) adversary $\adv$ should be able to distinguish between $\unlearnedmodel$ and $\controlmodel$. We also assume that $\adv$ has access to $\origmodel$, $\learn$, $\unlearn$, $\trainset$, and $\forgetset$.
\section{Formalizing Unlearning}
\label{sec:game}

We propose \emph{computational machine unlearning} as a formal way to capture that machine unlearning is indistinguishability. Unlike prior machine unlearning scores, our definition is defined as a security game, inspired by the cryptographic notion of semantic security and indistinguishability under chosen plaintext attack (IND-CPA) \autocite{2023Boneh_Textbook}. Instead of considering a membership inference score or accuracy gap, computational unlearning considers the ability of an adversary to distinguish between an unlearned model and a control model.

\subsection{Preliminaries}

Let $\datauniverse$ be the universe of all possible data, and $d \in \datauniverse$ be a particular data point. Let $\trainset \subseteq \datauniverse$ be our entire training dataset with $\forgetset \subseteq \trainset$ be the forget set. Let $\hypspace$ be our hypothesis space of possible models, with $h \in \hypspace$ being a particular model.

\begin{definition}[Learning scheme]
    We formally define a \emph{learning scheme} as a tuple of probabilistic polynomial time (p.p.t.) algorithms $(\modelinit, \learn, \infer)$:
    \begin{itemize}
        \item $\modelinit(1^\lambda) \rightarrow h$: randomly samples some initial model $h$. The notation $1^\lambda$ simply denotes that there are $\lambda$ copies of the symbol $1$ written on the input tape of the Turing machine and $0$ in every other location. This ensures that $\modelinit$ runs in polynomial time with respect to $\lambda$, a cryptographic formality.
        \item $\learn(h, \trainset) \rightarrow h$: given some initial model $h$, performs some model update process with respect to the training set $\trainset$.
        \item $\infer(h, d) \rightarrow \mathbb{R}^n$: performs some inference procedure with the given model $h$ on the provided data point $d$. 
    \end{itemize}
\end{definition}

\begin{remark}[Baseline and meaningful utility]
A newly initialized model under some learning scheme (i.e.\ the output of $\modelinit$) is expected to only have some baseline utility. Formally, we say the following:

$$\mathbb{E}(\util(\modelinit(1^\lambda))) = b$$

\noindent for some $b$, where $\modelinit(1^\lambda)$ initializes a model and does no training on it. We say utility is {\em meaningful} if it is larger than $b$.
\end{remark}

\begin{definition}[Forgetting learning scheme]
    We likewise define a \emph{forgetting learning scheme} as a tuple of p.p.t algorithms $(\modelinit, \learn, \infer, \unlearn)$ such that it is a learning scheme with an additional $\unlearn$ algorithm:
    \begin{itemize}
        \item $\unlearn(h, \forgetset) \rightarrow h$: performs some model update process with respect to the forget set $\forgetset$. 
    \end{itemize}
\end{definition}

\begin{remark}[Cost and utility functions]
We also assume the existence of cost and utility functions for learning schemes and forgetting learning schemes.

$$\cost: \Phi \rightarrow \mathbb{R}^+$$
$$\util: \hypspace \rightarrow \mathbb{R}^+$$

\noindent where $\cost$ measures the expense of performing a model update algorithm (e.g.\ $\learn,\ \unlearn$) and $\util$ is a function that measures the performance of a particular model in the hypothesis space.
\end{remark}

\begin{definition}[Negligible function]
We define $\mathrm{negl}(\lambda)$ to be a function that is \emph{negligible} in terms of a security parameter $\lambda$. We borrow the definition of a negligible function from cryptography --- namely, that a function $f : \mathbb{Z}_{\geq 1} \rightarrow \mathbb{R}$ is negligible if and only if for all $c > 0$ we have:

$$\lim\limits_{n \rightarrow \infty} f(n) n^c = 0$$

\noindent For example, if a function is bounded by $2^{-\lambda}$ then that function is negligible with respect to $\lambda$.
\end{definition}

\subsection{Computational Unlearning}

We now formally define {\em computational unlearning} in both white-box and black-box settings. As described earlier, these definitions are inspired by the cryptographic notions of semantic security and indistinguishability under chosen plaintext attack (IND-CPA)~\autocite{2023Boneh_Textbook}. We assume a computationally bound adversary and allow a negligible adversary advantage in keeping with these formal cryptographic definitions.

\begin{figure}
    \centering
    \includegraphics[width=0.45\linewidth]{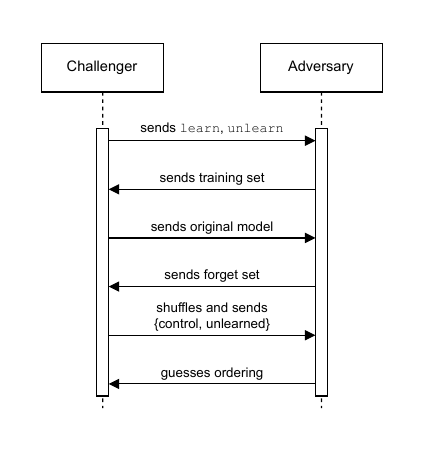}
    
    \caption{Overview of the security game for computational unlearning.}
    \label{fig:computational-unlearning}
\end{figure}

\begin{definition}[White-Box Computational Unlearning]
\label{def:white-box-computational-unlearning}
We consider the following experiment:
    \begin{enumerate*}
    \raggedright
        \item $\chal$ sends the description of the forgetting learning scheme (i.e.\ the $\learn$ and $\unlearn$ algorithms).
        \item $\adv$ chooses $\trainset$ and sends it to $\chal$.
        \item $\chal$ computes $\origmodel \leftarrow \learn(\modelinit(1^\lambda), \trainset)$ and sends $(\origmodel, \learn, \unlearn, \trainset, \cost, \util)$ to $\adv$.
        \item $\adv$ selects a forget set $\forgetset \subset \trainset$ and sends $\forgetset$ to $\chal$.
        \item $\chal$ computes $\unlearnedmodel \leftarrow \unlearn(\origmodel, \forgetset)$ and computes $\controlmodel \leftarrow \learn(\modelinit(1^\lambda), \retainset)$.
        \item $\chal$ samples a random bit $b \overset{{\$}}{\leftarrow} \{0, 1\}$. If $b = 0$, $\chal$ sends $[\controlmodel, \unlearnedmodel] $. If $b = 1$, $\chal$ sends $[\unlearnedmodel, \controlmodel]$.
        \item $\adv$ computes a guess $b'$ and sends $b'$ to $\chal$. $\adv$ wins the game if $b' = b$.
    \end{enumerate*}
    
    We say that an unlearning algorithm is a \emph{white-box computational machine unlearning algorithm} if
    
    $$\mathbb{P}\left(b' = b\right) < \frac{1}{2} + \mathrm{negl}(\lambda)$$

    \noindent We denote this computational indistinguishability by saying $\unlearnedmodel \compind \controlmodel$. This game is illustrated in~\autoref{fig:computational-unlearning}.
\end{definition}

\begin{definition}[Black-Box Computational Unlearning]
\label{def:black-box-computational-unlearning}
We consider the following experiment:
    \begin{enumerate*}
    \raggedright
        \item $\chal$ sends the description of the forgetting learning scheme (i.e.\ the $\learn$ and $\unlearn$ algorithms).
        \item $\adv$ chooses $\trainset$ and sends it to $\chal$.
        \item $\chal$ computes $\origmodel \leftarrow \learn(\modelinit(1^\lambda), \trainset)$ and sends $(\origmodel, \learn, \unlearn, \trainset, \cost, \util)$ to $\adv$.
        \item $\adv$ selects a forget set $\forgetset \subset \trainset$ and sends $\forgetset$ to $\chal$.
        \item $\chal$ computes $\unlearnedmodel \leftarrow \unlearn(\origmodel, \forgetset)$ and computes $\controlmodel \leftarrow \learn(\modelinit(1^\lambda), \retainset)$.
        \item $\chal$ samples a random bit $b \overset{{\$}}{\leftarrow} \{0, 1\}$. If $b = 0$, $\chal$ sends $[\mathcal{O}_{\controlmodel}, \mathcal{O}_{\unlearnedmodel}]$ where $\mathcal{O}$ is an oracle that allows $\adv$ to call $\infer$ on the underlying model. If $b = 1$, $\chal$ sends $[\mathcal{O}_{\unlearnedmodel}, \mathcal{O}_{\controlmodel}]$. 
        \item $\adv$ computes a guess $b'$ and sends $b'$ to $\chal$. $\adv$ wins the game if $b' = b$.
    \end{enumerate*}
    
    We say that an unlearning algorithm is a \emph{black-box computational machine unlearning algorithm} if we have:
    
    $$\mathbb{P}\left(b' = b\right) < \frac{1}{2} + \mathrm{negl}(\lambda)$$
\end{definition}

\begin{remark}[Threat Model]
    This definition intuitively captures a setting inspired by the GDPR process: we assuming the adversary is a user who can select which data should be deleted (i.e.\ the set of items to be deleted is, to some extent, adversarially-controlled) as in~\autocite{hu2024dutyforgetrightassured}. We also acknowledge that this game defines a very strong adversary and that a real-world adversary may not have access to the full training set, the description of the unlearning algorithm, or other information provided in this game. However, each of these alternatives envisions a strictly weaker adversary than our computational learning game, meaning that an unlearning method that achieves computational unlearning would still be indistinguishable from a control model in these scenarios.
\end{remark}

\begin{remark}[Trivial Solutions]
    Observe that there is a trivial solution for computational unlearning: simply defining $\unlearn$ to call $\modelinit(1^\lambda)$ and emit new models whose weights are initialized randomly. To prevent these trivial solutions, we require that $|\util(\origmodel) - \util(\controlmodel)| < \epsilon$ for some small $\epsilon$ and that $\cost\left(\learn\left(\retainset\right)\right) < \cost\left(\unlearn\left(\origmodel, \forgetset\right)\right)$.
\end{remark}

\begin{remark}[Utility Equivalence]
    We note that if the definition above is met, then $\util(\controlmodel) \compind \util(\unlearnedmodel)$ implicitly holds. If this was not the case, it would allow $\adv$ to distinguish the two models.
\end{remark}
\section{Empirical Results}
\label{sec:empirical}

We now present empirical distinguishers for $\adv$ to evaluate if unlearning methods from literature achieve computational unlearning. We experimentally demonstrate the effectiveness of these distinguishing algorithms on heuristic unlearning and certified removal methods.

\subsection{Distinguisher Scores}
\label{sec:scores}
Each distinguisher for $\adv$ uses a \emph{scoring function} to separate $\controlmodel$ from $\unlearnedmodel$.
The scoring function takes in the original model $\origmodel$, a candidate model $M \in \{M_1, M_2\}$, the training set $\trainset$, and the forget set $\forgetset$. The scoring function then outputs a value $s$ that is used to determine if the candidate model is $\unlearnedmodel$ or $\controlmodel$.

\paragraph{Scoring with membership inference attacks.}
\label{sec:miascore}
As described in \S\ref{sec:background}, membership inference attacks (MIA) are a common method for evaluating the performance of a given unlearning algorithm and several unlearning methods are justified by reducing them as much as possible. However, we are able to leverage these scores to distinguish an unlearned model from a control model \emph{because the unlearning method will often produce models whose MIA scores are out of distribution}. We propose that an unlearning algorithm should achieve similar MIA scores to a model that never saw the forget set rather than attempting to absolutely minimize it. In practice, we use the approach of Shokri et al.\ \autocite{2017Shokri_membership} for computing MIA scores using the same implementation as \autocite{2023FosterSSD}. We refer to this scoring algorithm as $\texttt{MIAScore}$.

\paragraph{Scoring with Kullback-Leibler divergence.}
\label{sec:kldscore}

We also present a novel scoring method \texttt{KLDScore}. We drew inspiration from the fact that Certified Removal bounds the KL-Divergence between different models. To calculate the score, $\adv$ calculates the KL-Divergence between the inferences of the original model $\origmodel$ and the candidate model $M$ (such as on instances in or near the forget set). This provides a measure of how different the behaviors of $M$ and $\origmodel$ are. In practice, we find that models produced by unlearning methods have much lower divergence from the original model than a control.

\begin{equation}
    \texttt{KLDScore}(\origmodel, M, \trainset, \forgetset) = \sum_{x_i \in \forgetset}{\kld(M(x_i + \mathcal{N}(0, 0.1))~\|~ \origmodel(x_i + \mathcal{N}(0, 0.1)))}
    \label{eqn:kldscore}
\end{equation}
\noindent where $\mathcal{N}(0, 0.1)$ represents Gaussian noise with mean $0$ and variance $0.1$.

\paragraph{Choice of decision rule.}
$\adv$ will compute $b'$ using the results from one of the aforementioned scoring algorithms. By Definitions \ref{def:white-box-computational-unlearning} and \ref{def:black-box-computational-unlearning}, $\adv$ is free to use prior knowledge of $\learn, \unlearn, \trainset$, and $\forgetset$ in the decision rule. We refer the reader to Kerckhoffs's principle in cryptography.

\subsection{Experimental Results}
\label{sec:results}
We evaluate the distinguishers via their success rates in differentiating between $\unlearnedmodel$ and $\controlmodel$. We present our findings from two experiments: one varying the size of the forget set $\forgetset$ and the other varying the $\sigma$ parameter from Certified Deep Unlearning. Both experiments were run using an Intel Xeon Gold 6330 and a NVIDIA A40. All results are statistically significant (i.e.\ a 95\% confidence interval under a Beta distribution with the Jeffries prior does not contain 50\%).

\paragraph{Implementation Details}
\label{sec:params}
All models used the ResNet-18~\autocite{2015He_resnet} architecture. The original and control models were trained using stochastic gradient descent with momentum and weight decay. The hyperparameters used are as follows:
\begin{itemize}
    \item Number of epochs: $50$
    \item Batch size: $512$
    \item Learning rate: $10^{-2}$
    \item Weight decay: $5 \times 10^{-4}$
\end{itemize}

For SSD~\autocite{2023FosterSSD}, we used a dampening constant of $1$ and a selection weighting of $100$. For all other methods \autocite{2023Chundawat_bad_teaching, 2020Graves_Amnesiac, 2024Zhang_Certified_Deep_Unlearning}, we used the parameters specified in their original papers (with the exception of $\sigma$ for CDU~\autocite{2024Zhang_Certified_Deep_Unlearning}, which we varied as described below).

\paragraph{Forget set size.}
We evaluated the effect of the forget set size on four different unlearning techniques. We used three heuristic methods and the approximate technique Certified Deep Unlearning (CDU) all discussed in \S\ref{sec:background}.

For each method, a random subset of $\trainset$ was chosen as the forget set. We varied the forget set size to evaluate its effect on the ability of $\adv$ to distinguish between $M_u$ and $M_c$ and correctly guess $b'$ using the distinguishing algorithms discussed above. We ran 128 trials, each with a different randomly selected forget set. We found that with increased forget set size the adversary was able to correctly guess $b'$ with higher frequency, but always maintained above a 60\% success rate at every forget set size. As we hypothesized, many heuristic unlearning techniques over-minimized \texttt{MIAScore} during their process of unlearning: for all heuristic unlearning methods the decision rule assigns a lower \texttt{MIAScore} score to $\unlearnedmodel$ (except for SSD~\autocite{2023FosterSSD} with greater than 30 forget set examples).

\begin{figure*}[h]
    \begin{subfigure}{0.5\textwidth}
    \centering
    \begin{tikzpicture}
        \begin{axis}[
            width=\textwidth,
            height=8cm,
            xlabel={Size of Forget Set},
            ylabel={$\adv$ Success Rate},
            grid=major,
            log ticks with fixed point,
            legend style={
                at={(0.98,0.02)},
                anchor=south east,
            },
            ymin=.45,
            ymax=1.05,
            xmode=log,
            log basis x=10,
        ]
        \addplot[
            color=blue,
            mark=*,
        ] table[
            col sep=comma,
            x=forget_set_size,
            y=win_percentage
        ] {
            forget_set_size,win_percentage
            3,0.984375
            6,0.9921875
            30,1.0
            60,1.0
            300,1.0
            600,1.0
            3000,1.0
        };
        \addlegendentry{Amnesiac}
        \addplot[
            color=green,
            mark=*,
        ] table[
            col sep=comma,
            x=forget_set_size,
            y=win_percentage
        ] {
            forget_set_size,win_percentage
            3,0.8203125
            6,0.9296875
            30,1.0
            60,1.0
            300,1.0
            600,1.0
            3000,1.0
        };
        \addlegendentry{Bad Teacher}
        \addplot[
            color=red,
            mark=*,
        ] table[
            col sep=comma,
            x=forget_set_size,
            y=win_percentage
        ] {
            forget_set_size,win_percentage
            3,0.6640625
            6,0.671875
            30,0.640625
            60,0.890625
            300,1.0
            600,1.0
            3000,1.0
        };
        \addlegendentry{SSD}
        \addplot[
            color=black,
            mark=*,
        ] table[
            col sep=comma,
            x=forget_set_size,
            y=win_percentage
        ] {
            forget_set_size,win_percentage
            3,0.9375
            6,0.96875
            30,1.0
            60,1.0
            300,1.0
            600,1.0
            3000,1.0
        };
        \addlegendentry{CDU}
        \legend{}
        \end{axis}
        \end{tikzpicture}
    \caption{\texttt{KLDScore}}
    \end{subfigure}
    \begin{subfigure}{0.5\textwidth}
        \begin{tikzpicture}
            \begin{axis}[
                width=\textwidth,
                height=8cm,
                xlabel={Size of Forget Set},
                grid=major,
                log ticks with fixed point,
                legend style={
                    at={(0.98,0.02)},
                    anchor=south east,
                },
                ymin=.45,
                ymax=1.05,
                xmode=log,
                log basis x=10,
                yticklabel=\empty,
            ]
            \addplot[
                color=blue,
                mark=*,
            ] table[
                col sep=comma,
                x=forget_set_size,
                y=win_percentage
            ] {
                forget_set_size,win_percentage
                3,0.70703125
                6,0.83984375
                30,0.96875
                60,0.98828125
                300,1.0
                600,1.0
                3000,1.0
            };
            \addlegendentry{Amnesiac}
            \addplot[
                color=green,
                mark=*,
            ] table[
                col sep=comma,
                x=forget_set_size,
                y=win_percentage
            ] {
                forget_set_size,win_percentage
                3,0.8515625
                6,0.96875
                30,1.0
                60,1.0
                300,1.0
                600,1.0
                3000,1.0
            };
            \addlegendentry{Bad Teacher}
            \addplot[
                color=red,
                mark=*,
            ] table[
                col sep=comma,
                x=forget_set_size,
                y=win_percentage
            ] {
                forget_set_size,win_percentage
                3,0.66015625
                6,0.6953125
                30,0.6796875
                60,0.734375
                300,1.0
                600,1.0
                3000,1.0
            };
            \addlegendentry{SSD}
            \addplot[
                color=black,
                mark=*,
            ] table[
                col sep=comma,
                x=forget_set_size,
                y=win_percentage
            ] {
                forget_set_size,win_percentage
                3,0.65625
                6,0.578125
                30,0.921875
                60,0.984375
                300,1.0
                600,1.0
                3000,1.0
            };
            \addlegendentry{CDU}
            \end{axis}
    \end{tikzpicture}
    \caption{\texttt{MIAScore}}
    \end{subfigure}
    \caption{Forget set size against adversary success rate using \texttt{KLDScore} and \texttt{MIAScore} distinguishers.}
    \label{fig:forgetset-results}
\end{figure*}
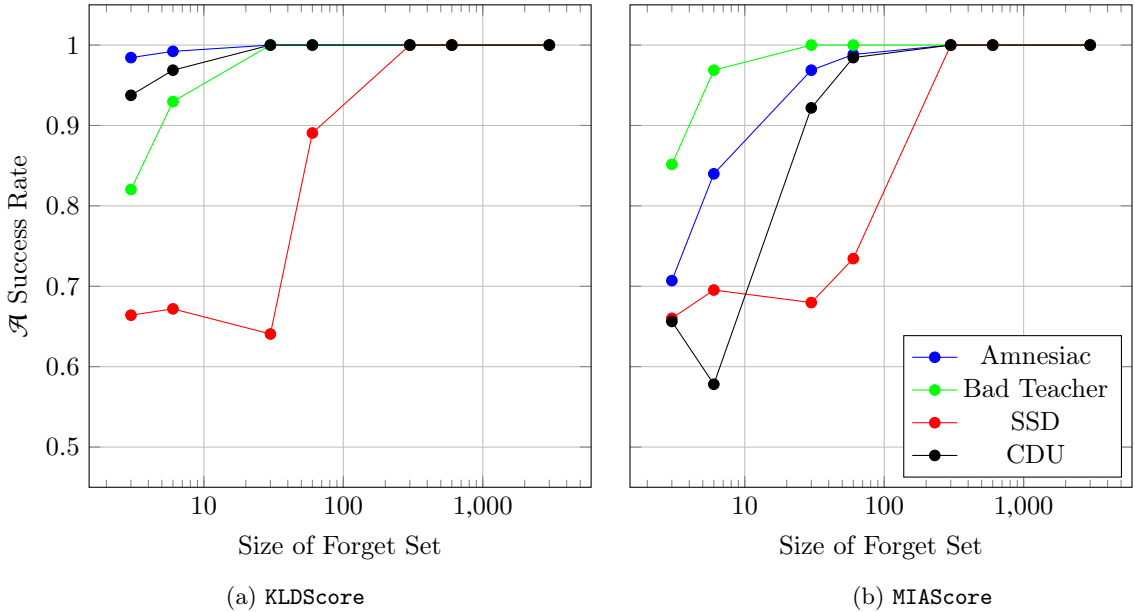

We also explored \emph{classwise} unlearning, where an entire class in $\trainset$ is chosen as the forget set. We found it always possible to distinguish in this setting (e.g.\ 100\% adversary success rate under both distinguishers). This is unsurprising given our results on the impact of forget set size. Recall that CIFAR-10 has 50,000 images in the training set, distributed evenly across 10 classes; forgetting an entire class amounts to a forget set size of 5,000 \autocite{2009Krizhevsky_cifar}.

\paragraph{Dependence on $\sigma$.}
We additionally explored the relationship between computational unlearning and certified removal's privacy parameters. For this we examined $\adv$'s \texttt{KLDScore} for certified deep unlearning (CDU) from Zhang et al.\ \autocite{2024Zhang_Certified_Deep_Unlearning} with different hyperparameters. The CDU method is based on a single hyperparameter $\sigma$, derived from $\epsilon$ and $\delta$ values, that represents the magnitude of noise used. We follow the hyperparameters from the CDU published experiments \autocite{2024Zhang_Certified_Deep_Unlearning}, including a random forget set of 1000 data points. We then varied $\sigma$ from $10^{-5}$ to $10^{-1}$ in powers of 10 running 128 trials at each value. 

\begin{figure*}[h]
    \centering
    \begin{tikzpicture}
        \begin{loglogaxis}[
            width=\textwidth,
            height=8cm,
            xlabel={$\sigma$},
            ylabel={\texttt{KLDScore}},
            grid=major,
            legend style={
                at={(0.02,0.98)},
                anchor=north west,
            },
            xmode=log,
            ymode=log,
            log basis x=10,
            log basis y=10,
        ]
        \addplot[
            color=blue,
            mark=*,
        ] table[
            col sep=comma,
            x=sigma,
            y=score
        ] {
            score,sigma
            0.34789903461933100,0.00001
            0.34789903461933100,0.0001
            0.34789903461933100,0.001
            0.34789903461933100,0.01
            0.34789903461933100,0.1
        };
        \addlegendentry{Control (Retrain)}
        \addplot[
            color=red,
            mark=*,
        ] table[
            col sep=comma,
            x=sigma,
            y=score
        ] {
            score,sigma
            0.0001117897413677900,0.00001
            0.00016919532754395900,0.0001
            0.005248369881883210,0.001
            6.343328133225440,0.01
            2086940021096448.0,0.1
        };
        \addlegendentry{Certified Deep Unlearning}
        \end{loglogaxis}
    \end{tikzpicture}
    
    \caption{Certified Deep Removal against \texttt{KLDScore} for different values of $\sigma$.} 
    \label{fig:cdu-sigma-kld}
\end{figure*}
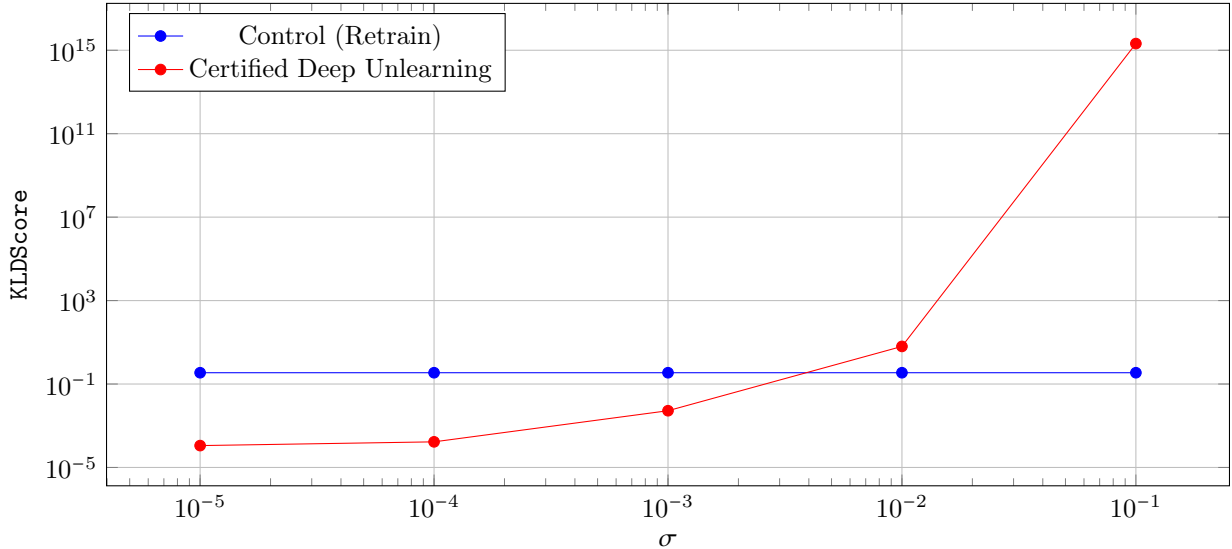

In our experiments we found that the adversary was able to distinguish using $\texttt{KLDScore}$ with 100\% accuracy for all choices of $\sigma$. We found as $\sigma$ increases the unlearned model's \texttt{KLDScore} also increases (see \autoref{fig:cdu-sigma-kld}). Since the control model has no dependency on $\sigma$, an adversary can distinguish with extremely high success rate by choosing a decision rule appropriate for the chosen value of $\sigma$. This relationship does imply there is a point of intersection (between 0.001 and 0.01) where the \texttt{KLDScore} score for $\unlearnedmodel$ and $\controlmodel$ should be very close, making it harder to distinguish using \texttt{KLDScore}. We believe understanding the intersection constitutes an interesting topic for future work.
\section{Theoretical Analysis}
\label{sec:theory}

We now show several interesting consequences of our computational unlearning definition. We begin by showing that $k$-NN admits a white-box computational unlearning algorithm in line with the technical intuition from \S\ref{sec:background}.

\begin{theorem}[$k$-NN admits white-box computational unlearning]
    \label{thm:knn-computational-unlearning}
    There is an efficient white-box computational unlearning algorithm for $k$-NN models.
\end{theorem}

\begin{proof}[Proof of Theorem \ref{thm:knn-computational-unlearning}]
    Let $\learn$ be defined as normal for $k$-NN models. Let $\unlearn$ be defined as deleting the specified $\forgetset$ from the $k$-NN database. Observe that this produces the same database as $\learn$ on $\retainset$. Therefore, an adversary cannot distinguish between $\unlearnedmodel$ and $\controlmodel$ with non-negligible advantage because they are exactly the same model.
\end{proof}

We first show that for entropic machine learning algorithms (e.g.\ stochastic gradient descent) there are no deterministic algorithms that can achieve computational unlearning. This result means that many heuristic unlearning methods can never admit computational unlearning algorithms. Secondly, we show that differentially private algorithms can achieve computational unlearning at the cost of collapsing model utility.

\subsection{Deterministic Computational Unlearning does not exist}
\label{sec:no-deterministic-computational-unlearning}

\paragraph{Preliminaries.}
We now show that a forgetting learning scheme that is entropic must have a randomized unlearning algorithm. Additionally, we show that a forgetting learning scheme that is deterministic must achieve perfect unlearning. Because forgetting learning schemes that are entropic must must be randomized and because forgetting learning schemes that are deterministic must be perfect, we say that \emph{deterministic computational learning does not exist}.

Before beginning, we define \emph{entropic learning schemes} and \emph{perfect unlearning}.

\begin{definition}[Deterministic learning scheme]
    A learning scheme is deterministic if the distribution of models produced by $\learn\left(\modelinit\left(1^\lambda\right), \trainset\right)$ has Shannon entropy of $0$.
\end{definition}

\begin{definition}[Entropic learning scheme]
    A learning scheme is \emph{$h$-entropic} if the distribution of models produced by $\learn\left(\modelinit\left(1^\lambda\right), \trainset\right)$ has Shannon entropy greater than or equal to $h$. In the absence of a particular value specified for $h$, we take $h$ to be $1$ bit.
\end{definition}

\begin{remark}
    If a learning scheme is entropic, it cannot be deterministic. For all practical purposes, learning schemes are either deterministic (i.e.\ $k$-nearest neighbors) or entropic (i.e.\ randomly initialized neural nets trained under stochastic gradient descent).
\end{remark}

\begin{definition}[Perfect unlearning]
\label{def:perfect_unlearning}
    We say a forgetting learning scheme achieves \emph{perfect unlearning} algorithm if, for all $M = \learn(\modelinit(1^\lambda), \trainset)$, the following always holds:
    $$\learn(\modelinit(1^\lambda), \retainset) = \unlearn(M, \forgetset)$$
    This is to say, \unlearn\ is perfect if it produces \emph{exactly the same} model as retraining on the retain set.
\end{definition}

\begin{remark}
    A perfect unlearning algorithm is distinct from the notion of exact unlearning described in \S\ref{sec:background}. Literature that explores exact unlearning rewinds the learning process to the first step that used items from the forget set; in other words, it implements the below:
    $$\learn\left(\learn\left(\modelinit\left(1^\lambda\right), \retainset\right), \retainset\right)$$
\end{remark}

Recall that in our definition, $\adv$ is given $\unlearn$, the description of the unlearning method, and also has access to the original model $\origmodel$. Intuitively, this means that an adversary can simply \emph{run the unlearning method on its own}.

Because the unlearning algorithm is deterministic and the learning scheme is entropic, this means that only one of the two models will exactly match the adversary's own computed result with high probability and allow the adversary distinguish with non-negligible probability.

We now prove the above for entropic learning schemes that are forgetting and achieve computational unlearning.

\begin{theorem}
\label{thm:deterministic-entropic}
There are no deterministic computational unlearning algorithms for entropic learning schemes.
\end{theorem}

\begin{proof}[Proof of Theorem \ref{thm:deterministic-entropic}]
Suppose that a forgetting learning scheme is entropic. Therefore, $\learn\left(\modelinit\left(1^\lambda\right), \trainset\right)$ is a randomized algorithm that samples some $h \in \hypspace$ with minimum entropy greater than 1 bit. Let $\mathbb{P}(h)$ be the probability that $\learn$ samples a particular $h \in \hypspace$ and let 
$$p_{\mathrm{max}} = \underset{\forall h \in \mathcal{H}}{\max} \mathbb{P}(h)$$

Now suppose that the challenger uses a deterministic $\unlearn$ algorithm. Then the adversary can also run $\unlearn$ on $\origmodel$ and will win the game if $\controlmodel \neq \unlearnedmodel$. Because the probability $\learn$ will output a particular model is bounded by $p_{\mathrm{max}}$, the probability that $\controlmodel = \unlearnedmodel$ is also bounded by $p_{\mathrm{max}}$ and the probability $\controlmodel \neq \unlearnedmodel$ is at least $1 - p_{\mathrm{max}}$. Because $\unlearn$ is a computational unlearning algorithm, we must have that $1 - p_{\mathrm{max}} < \frac{1}{2} + \mathrm{negl}(\lambda)$. We can rearrange symbols to get that $\mathrm{negl}(\lambda) > \frac{1}{2} - p_{\mathrm{max}}$. But we have a contradiction because $p_{\mathrm{max}}$ does not asymptotically approach $\frac{1}{2}$ as $\lambda$ approaches infinity.
\end{proof}

We now show that a forgetting learning scheme that is deterministic and achieves computational unlearning must be perfect. The intuition for this result is similar: the adversary has access to $\learn$, the description of the learning algorithm, and has access to the $\retainset$. This means that the adversary can compute the control model on their own, use its own control model to identify the control model provided by the challenger, and distinguish with non-negligible probability.

\begin{theorem}
\label{thm:deterministic-perfect}
Let $\mathcal{L}$ be a forgetting learning scheme that is deterministic. Then if it satisfies the computational unlearning notion of Definitions \ref{def:white-box-computational-unlearning} and \ref{def:black-box-computational-unlearning} it must perfectly unlearn under Definition \ref{def:perfect_unlearning}.
\end{theorem}

\begin{proof}[Proof of \ref{thm:deterministic-perfect}]
Suppose that $\mathcal{L}$ is a deterministic learning scheme. Therefore, it must output a single model for a given training set $\trainset$. Suppose $\mathcal{L}$ is also forgetting and achieves computational unlearning. We now consider two possible cases: that $\unlearn$ is randomized and that it is deterministic.

\begin{itemize}
    \item {
        \textit{Randomized case:} Suppose that $\unlearn$\ is a randomized algorithm that samples some $h \in \hypspace$. Let $\mathbb{P}(h)$ be the probability that $\unlearn$ selects a particular $h \in \hypspace$ and let 
        
        $$p_{\mathrm{max}} = \underset{\forall h \in \mathcal{H}}{\max} \mathbb{P}(h)$$
    
        Recall that in this scenario, the challenger uses a deterministic $\learn$ algorithm to produce $M_c$. Then the adversary can also run $\learn$ to produce $\controlmodel$ and will win the game if $\controlmodel \neq \unlearnedmodel$. Because the probability $\unlearn$ will output a particular model is bounded by $p_{\mathrm{max}}$, the probability that $\controlmodel = \unlearnedmodel$ is also bounded by $p_{\mathrm{max}}$ and the probability $\controlmodel \neq \unlearnedmodel$ is at least $1 - p_{\mathrm{max}}$. Because $\unlearn$ is a computational unlearning algorithm, we must have that $1 - p_{\mathrm{max}} < \frac{1}{2} + \mathrm{negl}(\lambda)$. We can rearrange symbols to get that $\mathrm{negl}(\lambda) > \frac{1}{2} - p_{\mathrm{max}}$. But we have a contradiction because $p_{\mathrm{max}}$ does not asymptotically approach $\frac{1}{2}$ as $\lambda$ approaches infinity.
    }

    \item {
        \textit{Deterministic case:} Now suppose that $\unlearn$ is a deterministic algorithm. Then the adversary can also run $\learn$ and $\unlearn$ on $\origmodel$ and will win the game if $\controlmodel \neq \unlearnedmodel$. Because $\learn$ and $\unlearn$ are deterministic and will each output a particular model for a given dataset, we must have that $\controlmodel = \unlearnedmodel$. Thus, $\unlearn$ must be a \emph{perfect} unlearning algorithm.
}
\end{itemize}
\end{proof}

\begin{remark}[Viability of computational unlearning methods]
These results constrain the space of learning algorithms that are compatible with unlearning. To reiterate: Theorem \ref{thm:deterministic-entropic} shows that entropic learning schemes that are forgetting and achieve computational unlearning must have a randomized unlearning method. In the opposite direction, no deterministic learning algorithms can support entropic unlearning algorithms. Any deterministic learning scheme that is forgetting and achieves computational unlearning must implicitly realize a \emph{perfect} unlearning scheme, as noted in Theorem \ref{thm:deterministic-perfect}. As a consequence of these findings, any forgetting learning schemes that achieves computational unlearning must either be perfect, or both the learning and unlearning process must inherently be randomized. Note that Certified Deep Unlearning~\autocite{2024Zhang_Certified_Deep_Unlearning} and many heuristic unlearning methods we studied in \S\ref{sec:empirical} are not randomized and are not perfect. Thus, they can never achieve computational unlearning.
\end{remark}

\subsection{Computational Unlearning from Differential Privacy Collapses Utility}
\label{sec:cu-dp}

One natural approach to constructing computational unlearning uses techniques from differential privacy~\autocite{2014dwork_dp}.

While differentially private learning algorithms imply the existence of black-box computational unlearning, the parameters choices required to achieve computational unlearning will lead to utility collapse for the resulting models. We claim that the $\epsilon$ and $\delta$ parameters must be phrased in terms of $\lambda$ and that values needed to obtain security imply unacceptably high utility loss.

We begin by recalling the definition of privacy loss and differential privacy.

\begin{definition}[Privacy Loss, \autocite{2014dwork_dp}]
\label{def:privacy-loss}
The privacy loss $\mathcal{L}$ over neighboring databases $x, y$ after observing $\xi$ is given by:

$$\mathcal{L}^{(\xi)}_{\mathcal{M}(x)\|\mathcal{M}(y)} = \ln\left(\frac{\mathbb{P}(\mathcal{M}(x) = \xi)}{\mathbb{P}(\mathcal{M}(y) = \xi)}\right)$$
\end{definition}

\begin{definition}[Differential Privacy, \autocite{2014dwork_dp}]
\label{def:statistical-differential-privacy}
A randomized algorithm $\mathcal{M}$ with domain $\mathbb{N}^{|\mathcal{X}|}$ is $(\epsilon, \delta)$-differentially private if for all $\mathcal{S} \subseteq \mathrm{Range}(\mathcal{M})$ and for all $x, y \in \mathbb{N}^{|\mathcal{X}|}$ such that $\|x-y\|_1 \leq 1$:

$$\mathbb{P}(\mathcal{M}(x) \in \mathcal{S}) \leq e^{\epsilon} \cdot 
\mathbb{P}(\mathcal{M}(y) \in \mathcal{S}) + \delta$$

If $\delta = 0$, we say that $\mathcal{M}$ is $\epsilon$-differentially private.
\end{definition}

Differential privacy's definition bounds the privacy loss from any query, which we discuss below.

\begin{remark}[Privacy Loss Bounded for Differentially Private Algorithms, \autocite{2014dwork_dp}]
Suppose that $\mathcal{M}$ is a $(\epsilon, \delta)$-differentially private algorithm. Then by definition, the absolute value of the privacy loss $\mathcal{L}^{(\xi)}_{\mathcal{M}(x)\ \|\ \mathcal{M}(y)}$ is bounded by $\epsilon$ with probability at least $1-\delta$.
\end{remark}

\begin{remark}[Differential Privacy is Immune to Post-Processing, \autocite{2014dwork_dp}]
\label{rem:post-processing}
Additionally, one of the most useful properties of differential privacy is that it is ``immune'' to post-processing. This means that there exists no algorithm that, given the output of a differentially-private function, can ``undo'' the differential privacy. We refer the reader to \autocite[Proposition~2.1]{2014dwork_dp} for the proof of this claim. 
\end{remark}

We will use this property to show that differential privacy can be used to satisfy the definition of black-box computational unlearning (Definition \ref{def:black-box-computational-unlearning}).

\begin{lemma}
\label{lem:privacy-loss-entropy}
Privacy Loss is an upper bound on relative entropy.
\end{lemma}

\begin{proof}[Proof of Lemma \ref{lem:privacy-loss-entropy}]
Recall the definition of relative entropy (Kullback-Leibler divergence) of probability distribution $Q$ with respect to $P$ \autocite{1951KLDivergence}:

\begin{equation}
    \label{eqn:kld}
    \kld\left(P\ \|\ Q \right) = \sum_{x \in \mathcal{X}}P(x)\ln\left(\frac{P(x)}{Q(x)}\right) 
\end{equation}

Now, suppose we have some randomized algorithm $\mathcal{M}$ with inputs $a, b$. Let $P, Q$ represent the output distributions of $\mathcal{M}(a), \mathcal{M}(b)$ respectively. Let $\mathcal{L}_\text{max}$ refer to the maximum privacy loss observed for any element $x$.

\begin{align*}
    (\ref{eqn:kld}) &= \sum_{x \in \mathcal{X}}P(x)\ln\left(\frac{\mathbb{P}(\mathcal{M}(a) = x)}{\mathbb{P}(\mathcal{M}(b) = x)}\right) \\
    &= \sum_{x \in \mathcal{X}}P(x)\mathcal{L}^{(x)}_{\mathcal{M}(a)\ \|\ \mathcal{M}(b)} \\
    &\leq \sum_{x \in \mathcal{X}}\mathcal{L}_\text{max} \\
\end{align*}

Because $P$ is a probability distribution, we have that $P(x) \in [0, 1]$. Then privacy loss is an upper bound because the relative entropy is equal to the privacy loss multiplied by $P(x)$ by definition. 
\end{proof}

Having reviewed important properties of differential privacy, we now show how to construct black-box computational unlearning (Definition \ref{def:black-box-computational-unlearning}) from differential privacy. There are two main ways to accomplish this: to use differential privacy directly or to aggregate the outputs of models in a differentially private way. The theorem below captures both of these cases.

\begin{theorem}[Differentially private computational unlearning]
\label{thm:differentially-private-computational-unlearning}
Let $\mathcal{L}$ be a forgetting learning scheme that achieves black-box computational unlearning. Let $\unlearn$ simply output the original model (with fresh randomness for the differentially private mechanism). Then $\learn$ and $\unlearn$ satisfy the definition of black-box computational unlearning (Definition \ref{def:black-box-computational-unlearning}) if and only if $\delta \leq \mathrm{negl}\left(\lambda\right)$ and let $\epsilon \leq \ln \left( 1 + \mathrm{negl}\left(\lambda\right)\right)$.
\end{theorem}

\begin{proof}[Proof of Theorem \ref{thm:differentially-private-computational-unlearning}]
Observe that the privacy loss is negligible in $\lambda$ with overwhelming probability. This means that the relative entropy between the outputs of $\unlearnedmodel$ and $\controlmodel$ is negligible by Lemma \ref{lem:privacy-loss-entropy}. By Remark \ref{rem:post-processing}, there is no algorithm an adversary can use to increase the relative entropy. So then $\unlearnedmodel$ and $\controlmodel$ are computationally indistinguishable.

We now show that our bounds are tight. Suppose that $\delta > \mathrm{negl}(\lambda)$. Then the privacy loss guarantee does not hold with overwhelming probability and an adversary could obtain a query result with non-negligible privacy loss after a polynomial number of queries.

Alternatively, suppose that $\epsilon > \ln \left( 1 + \mathrm{negl}\left(\lambda\right)\right)$. Then the privacy loss guarantee is at least polynomial in $\lambda$ and an adversary could obtain query results that lead to a non-negligible privacy loss after a polynomial number of queries.
\end{proof}

Unfortunately this approach also has the following undesirable result:

\begin{corollary}
\label{cor:utility-collapse}
Let $\mathcal{L}$ be a forgetting learning scheme that achieves black-box computational unlearning, with \learn\ implemented as described in Theorem 
\ref{thm:differentially-private-computational-unlearning}. Then $\unlearnedmodel$ and $\origmodel$ are also computationally indistinguishable. This implies that the utility of $\origmodel$ is equivalent to the utility of $\unlearnedmodel$.
\end{corollary}

\begin{proof}[Proof of Corollary \ref{cor:utility-collapse}]
We follow the proof of Theorem \ref{thm:differentially-private-computational-unlearning}. Observe that the privacy loss is negligible in $\lambda$ with overwhelming probability. This means that the relative entropy between the outputs of $\unlearnedmodel$ and $\controlmodel$ is negligible. But $\unlearnedmodel$ is the same model as $\origmodel$, with fresh randomness for the differential privacy mechanism. So $\unlearnedmodel$ and $\origmodel$ are also computationally indistinguishable.

In other words, this means that $\util\left(\origmodel\right) \compind \util\left(\unlearnedmodel\right)$. Since $\chal$ does not know {\em a priori} the choice of $\adv$, $\unlearn$ must be indistinguishable for all possible choices. So then $\unlearnedmodel \compind \controlmodel$ for $\forgetset = \trainset$. That is to say that $\unlearnedmodel \compind \learn\left(\modelinit\left(1^\lambda\right), \emptyset\right)$. But we because $\util(\origmodel) \compind \util(\unlearnedmodel)$ we also have $\util\left(\origmodel\right) \compind \util\left(\learn\left(\modelinit\left(1^\lambda\right), \emptyset\right)\right)$, which is bounded by a small $\epsilon$ and thus not meaningful.
\end{proof}

\begin{remark}[Black-box infeasibilty implies white-box infeasibility.]
The security notion of white-box computational unlearning in Definition~\ref{def:white-box-computational-unlearning} is strictly stronger than the black-box computational unlearning of Definition~\ref{def:black-box-computational-unlearning}. Thus, the an infeasibility result for black-box computational unlearning immediately implies an infeasibility result for white-box computational unlearning.
\end{remark}

We note that most use cases of differentially private $\infer$ algorithms are designed to support a number of queries bounded by a constant. One possible interpretation of our result is that we assume an adversary is able to query the model some polynomial number of times.

We additionally stress that Theorem~\ref
{thm:differentially-private-computational-unlearning} and Corollary~\ref{cor:utility-collapse} only consider applying differential privacy to the $\infer$ algorithm of a learning scheme. Our result does not necessarily imply a utility collapse for a forgetting learning scheme that achieves computational unlearning with a differentially private $\learn$ algorithm.
\section{Discussion}
\label{sec:discussion}

\paragraph{Several unlearning methods are deterministic.}
Several unlearning methods in the literature are deterministic. This means that an adversary will always be able to distinguish in the computational unlearning game (see \S\ref{sec:no-deterministic-computational-unlearning}). In particular, \autocite{2023Guo_Certified_Removal, 2024Zhang_Certified_Deep_Unlearning, 2023FosterSSD, 2020Graves_Amnesiac} each provide deterministic unlearning methods and will never achieve computational unlearning. This includes various certified removal and approximate techniques; while they are randomized during the $\learn$ process they are not randomized during the $\unlearn$ process.

In contrast to the above methods, \autocite{2016Abadi_DP_SGD, 2023Chundawat_bad_teaching} are randomized. Therefore, it is possible that they could achieve computational unlearning if other issues leading to distinguishing attacks are resolved.

See Table \ref{tbl:summary} for an overview.

\begin{table*}[t]
    \begin{center}
        \begin{threeparttable}
            \begin{tabular}{rccccccc} \toprule
                Name & Category & Randomized \\ 
                \midrule
                DP-SGD \autocite{2016Abadi_DP_SGD} & Approximate & \cmark \\ 
                Certified Removal \autocite{2023Guo_Certified_Removal, 2024Zhang_Certified_Deep_Unlearning} & Approximate & \xmark \\
                SSD \autocite{2020bourtoule_machine} & Heuristic & \xmark \\
                Amnesiac \autocite{2020Graves_Amnesiac} & Heuristic & \xmark \\
                Bad Teacher \autocite{2023Chundawat_bad_teaching} & Heuristic & \cmark \\ 
                \bottomrule
            \end{tabular}
        \vspace{-9em}
        \caption{Summary of the randomization of unlearning methods.}
        \label{tbl:summary}
        \end{threeparttable}
    \end{center}
\end{table*}

\paragraph{Bounding the difference between models does not imply indistinguishability.} As described in \S\ref{sec:empirical}, we compare (1) the KL divergence between $\origmodel$ and $\controlmodel$ and (2) the KL divergence between $\origmodel$ and the $\unlearnedmodel$ to distinguish certified removal models with advantage far higher than would be expected from the $\delta$ parameter. This may seem unintuitive at first. We diagram in \autoref{fig:certified-removal-distinguishing} how this is possible.

\begin{figure}[ht]
    \centering
    \includegraphics[width=0.5\linewidth]{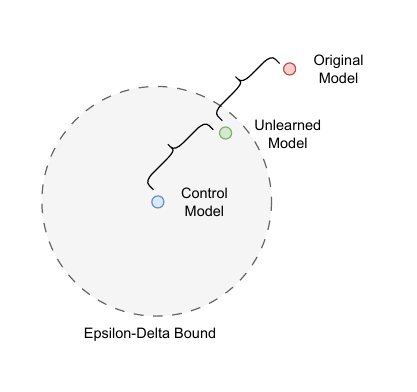}
    \caption{Intuition of how KL divergence is able to distinguish between $\unlearnedmodel$ from certified removal and $\controlmodel$. The distribution produced by the unlearned model is within the $\epsilon, \delta$ bound but the adversary is able to leverage access to the original model to distinguish between the original model and a control model.}
    \label{fig:certified-removal-distinguishing}
\end{figure}

While $\unlearnedmodel$ falls within the certified removal bound around $\controlmodel$, the certified removal bound does not guarantee that $\unlearnedmodel$ will be uniformly distributed within the bound. Because we assume the adversary also has knowledge of $\origmodel$ under our computational unlearning threat model, we can find the model that has the least divergence with the original model. Under other notions of distance, it may also be possible to distinguish based on the \emph{direction} of the differences---not just the magnitude of the differences.

As described in Theorem \ref{thm:differentially-private-computational-unlearning}, conventional choices for $\epsilon$ and $\delta$ are too loose to achieve computational unlearning. They must be phrased in terms of $\lambda$, the security parameter, with $\delta \leq \mathrm{negl}\left(\lambda\right)$ and $\epsilon \leq \ln \left( 1 + \mathrm{negl}\left(\lambda\right)\right)$. This ensures that the privacy loss of every query is negligible. After a polynomial number of queries, the adversary will still have negligible information and thus will have negligible advantage.

When not set properly, the bound is loose enough to permit distinguishing as empirically demonstrated in \S\ref{sec:empirical}.

\subsection{Future work}

\paragraph{Relaxations of our definition.} As described in \autocite{2014dwork_dp}, the Fundamental Law of Information Recovery states that ``overly accurate answers to too many questions will destroy privacy in a spectacular way.'' For this reason, most research into differential privacy considers a bounded query model; there is at most some fixed number of queries that the adversary can make. We claim that this is impractical for unlearning. This would still require complete retraining of the model and avoiding this is the entire point of unlearning.

While we do not believe that a constant of queries is a suitable unlearning scenario, we believe there are various other relaxations that may prove useful. For example, it may be possible to let the challenger delete additional information beyond what is selected by the adversary. It is unclear if this would provide meaningful realizations of unlearning, but is more in line with~\autocite{2020bourtoule_machine}. Another possible relaxation could be to allow for some bounded, non-negligible adversary advantage.

Alternative relaxations could include giving the adversary less information, such as not giving them access to the original model ($\origmodel$). Our analysis in \S\ref{sec:theory} depends on this information being available to the adversary. In cases where this information is not available, certified removal with a sufficiently tight bound may be a viable alternative.

\paragraph{Alignment of generative models via unlearning.} This work focuses on image classification models to facilitate iterative experimentation but our definition naturally extends to generative models, including large language models (LLMs). In this context, unlearning can be viewed as an alignment technique. However, this presents additional challenges. 

Language is a discrete sequence-based modality where changes to few tokens in the sequence can cause massive semantic changes to the meaning. Similarly, drastically different input sequences can contain the exact same information. This can make it hard to specify data points that contain the information you wish to forget. It is also unclear if certain concepts are \emph{emergent}, meaning that the model can infer them even without explicit training data. For example, if you wish to remove all the information from a generative model about weapons, it is not as simple as forgetting all data points that contain the word ``weapon.'' This is further complicated because language datasets are often scraped from many different sources, meaning there are many possible sources of information that a user may wish to remove.

Accurately specifying all data points that contain the relevant information is non-trivial and an open area of research. It may be possible to use an embedding of some kind to determine semantic similarity, but the effect this has on downstream models is an open area of research.

\paragraph{Fine-tuning to foundation models.} A common technique is to fine-tune large foundation models (e.g.\ LLMs, diffusion models) to a particular task. It is natural to rephrase our definitions of $\learn$ and $\unlearn$ to be fine-tunings of a foundation model. Because the adversary already has foreknowledge of the foundation model (they have the description of $\learn$) the adversary is only distinguishing the results of the fine-tuning process from a control model.
\section{Conclusion}

In summary, we have proposed computational unlearning, a new framework for evaluating machine unlearning. Computational unlearning is satisfied by an unlearning method if the output of the unlearning method is indistinguishable from a mirror (control) model. We rigorously define indistinguishability in terms of a novel two-party cryptographic protocol which captures an adversary's ability to distinguish between two models. Computational unlearning provides both empirical and theoretical contributions to the field of unlearning by improves upon prior evaluation methods, such as membership inference attack (MIA) scores.

We empirically showed that several machine unlearning methods from literature \autocite{2023FosterSSD, 2020Graves_Amnesiac, 2023Chundawat_bad_teaching, 2024Zhang_Certified_Deep_Unlearning} do not achieve computational unlearning by presenting multiple algorithms that allow an adversary to distinguish between the model produced by an unlearning method and a control model. 

We have identified several theoretical implications that naturally follow from our formal definition of computational unlearning. For example, all unlearning methods that meet our definition of computational unlearning must be randomized; there are no deterministic computational unlearning methods despite there being several deterministic unlearning methods proposed in prior work. We also proved that building computational machine unlearning using differential privacy techniques leads to utility collapse.

Lastly, we outlined various directions for future work. This includes implementing high-utility general-purpose computational unlearning, potential relaxations of our computational unlearning framework, using unlearning to align generative models, and exploring how to incorporate notions of foundation models into computational unlearning.

\printbibliography

\end{document}